\documentclass[letterpaper, 10 pt, conference]{ieeeconf} 
\IEEEoverridecommandlockouts
\newcommand{\hatxi}{\hat{\xi}}

\newcommand{\hatp}{\hat{\mathcal{P}}}
\newcommand{\hatpi}{{\hat{\pi}}}
\newcommand{\muxi}{\hat{\mu}^{\xi}}
\newcommand{\muhat}{{\hat{\mu}}}
\newcommand{\sigmahat}{{\hat{\sigma}}}
\newcommand{\barp}{\bar{\mathbb{P}}_{k+l|k}^\theta}

\newcommand{\cov}{\text{cov}}
\newcommand{\tr}{\text{Tr}}

\newcommand{\cvar}{\textrm{CVaR}}
\newcommand{\drcvar}{\textrm{DR-CVaR}}

\usepackage[utf8]{inputenc}
\usepackage{cite}
\usepackage{hyperref}

\usepackage{amsmath,amssymb,amsfonts,amsthm}
\usepackage{tabularx}
\usepackage{float}
\usepackage{multirow}
\usepackage{cuted}
\usepackage{booktabs}

\usepackage[linesnumbered, ruled, vlined]{algorithm2e}
\usepackage{setspace}
\usepackage{graphicx}
\graphicspath{ {./figures/} }
\usepackage{textcomp}
\usepackage{xcolor}
\usepackage{multicol}
\newtheorem{theorem}{Theorem}[section]
\newtheorem{lemma}[theorem]{Lemma}
\newtheorem{proposition}[theorem]{Proposition}

\newtheorem{definition}{Definition}

\title{\LARGE \bf  Addressing Behavior Model Inaccuracies for Safe Motion Control \\ in Uncertain Dynamic Environments}

\author{Minjun Sung, Hunmin Kim, and Naira Hovakimyan
\thanks{Minjun Sung and Naira Hovakimyan are with the Department of Mechanical Science and Engineering, University of Illinois at Urbana-Champaign, USA. {\tt\small\{mjsung2, nhovakim\}@illinois.edu}. Hunmin Kim is with the Department of Electrical and Computer Engineering, Mercer University, USA. {\tt\small kim$\_$h@mercer.edu}. }
\thanks{This work was supported by NASA ULI (\#80NSSC17M0051), AFOSR (\# FA9550-21-1-0411) and NSF (\#2135925, \#2331878).}%
\thanks{Project page can be found at \href{http://sied-mpc.notion.site}{http://sied-mpc.notion.site}.}}

\begin{document}
\maketitle

\begin{abstract}
    Uncertainties in the environment and behavior model inaccuracies compromise the state estimation of a dynamic obstacle and its trajectory predictions, introducing biases in estimation and shifts in predictive distributions. Addressing these challenges is crucial to safely control an autonomous system. In this paper, we propose a novel algorithm \texttt{SIED-MPC}, which synergistically  integrates Simultaneous State and Input Estimation (SSIE) and Distributionally Robust Model Predictive Control (DR-MPC) using model confidence evaluation. The SSIE process produces unbiased state estimates and optimal \emph{input gap} estimates to assess the confidence of the behavior model, defining the ambiguity radius for DR-MPC to handle predictive distribution shifts. This systematic confidence evaluation leads to producing safe inputs with an adequate level of conservatism. Our algorithm demonstrated a reduced collision rate in autonomous driving simulations through improved state estimation, with a 54\% shorter average computation time. 
\end{abstract}

\section{Introduction}
Ensuring the safety of mobile robots and Autonomous Vehicles (AV) in uncertain environments with dynamic obstacles is a paramount challenge. While the risk of uncertainties can be reduced with accurate \emph{estimation} and \emph{prediction}, integration between these fields for safe motion control remains insufficient. Specifically, estimation methods often assume control input to be directly accessible or the controller to be known~\cite{urrea2021kalman}, while predictions propagate a behavior model through obstacle dynamics assuming a reliable estimation is given~\cite[Sec.~3]{huang2022survey}. This circular assumption suggests a gap in the integration of estimation and prediction, limiting the potential for achieving truly safe robotic and AV systems.

The central challenge in integrating estimation and prediction lies in their reliance on a behavior model of an obstacle. A flawed behavior model not only biases the state estimator but also causes shifts in the predictive distribution. Furthermore, biases introduced by the estimation process itself can further exacerbate shifts in this distribution. Consequently, behavioral modeling techniques for accurate estimation and predictions have attracted significant interest for robots and AVs. Recent advancements in machine learning, particularly in processing sequential data, have substantially improved the reliability of predictive state sequences~\cite[Sec.~5]{huang2022survey}. However, even the most advanced models encounter corner cases and residual inaccuracies, resulting in discrepancies between the predicted and actual obstacle behaviors~\cite{ding2023survey,sungrobust}. This raises the following questions: 
\begin{enumerate}
    \item How can we systematically enhance the robustness of both estimation and prediction processes against errors in behavioral modeling?
    \item What advantages could be derived from addressing these processes in a unified manner? 
\end{enumerate} 

In this work, we introduce a novel framework that synergistically combines Simultaneous State and Input Estimation (SSIE) and Distributionally Robust Model Predictive Control (DR-MPC) to address the above questions. Specifically, the SSIE process enables an unbiased estimation of an obstacle state by explicit estimation of the input gap induced by the behavior modeling error. This approach allows us to locally assess the \emph{confidence} of the behavior model, which we leverage to compute an adequate level of conservatism in producing an obstacle's distributive forecasts. The integrated controller, termed S\texttt{SIE} and confidence-based \texttt{D}istributionally Robust \texttt{MPC} (\texttt{SIED-MPC}), is presented as a computationally tractable algorithm that effectively mitigates safety concerns arising from inaccuracies in the behavior model. The efficacy of our algorithm - accurate state and input estimation, reduced collision rate, and controlled conservatism - is examined through AV traffic simulation.

\subsection{Related Works}
\subsubsection{Estimating States with Unknown Inputs}
Input and State Estimation (ISE) has attracted attention within cyber-physical systems (CPS) to simultaneously estimate the input and state of a system~\cite{yong2016unified,kim2020data,wan2019attack}. However, the implementations and implications of these works have not been extended to address input gaps or predictive distribution shifts, which is the aim of our research. Earlier, the Interacting Multiple Model (IMM) was introduced in~\cite{kaempchen2004imm} to account for policy error through a combined weighted estimate, however, it relies on heuristic policy set choices. Similar Kalman Filter (KF)-based estimation was used in~\cite{zhang2017method}, however, it relied on vehicle-to-vehicle communication.

\subsubsection{Addressing Distribution Shifts in Trajectory Prediction}
Optimization problems to address distributional uncertainty have been widely studied with convex risk measures such as Conditional Value at Risk (CVaR)~\cite{kuhn_wasserstein_2019,wiesemann2014distributionally}. Recent works have employed this approach to solve safe-motion control problems in learning-enabled environments to address learning errors~\cite{hakobyan2021wasserstein,shetty2023safeguarding}. Our work is closely aligned with~\cite{9682981}, which uses Gaussian Process Regression (GPR) to learn the obstacle behavior model. Our method generalizes this approach by being agnostic to the choice of behavior model. Unlike previous efforts, our approach neither assumes accurate obstacle states nor heuristically determines the size of an ambiguity set.

\section{Problem Description}
\subsection{Problem Formulation}
Consider the dynamics of an ego agent represented in a discrete-time nonlinear system of the form
\begin{equation}\label{eq:robot-dynamics}
    x_{k+1} = f(x_k,u_k),\quad k=0,1,2\ldots
\end{equation}
with known $x_0$, where $x_k\in\mathbb{R}^{n_x}$ and $u_k\in \mathbb{U}_k\subset \mathbb{R}^{n_u}$ are the agent state and control input at $k$-th time step. Here, $\mathbb{U}_k$ is a compact set that defines control constraints. The vector field $f$ describes the known dynamics of a controlled agent.

The dynamics of an obstacle are given as:
\begin{subequations}\label{eq:obstacle-dynamics}
\begin{align}
    \xi_k = g(\xi_{k-1}) &+ B_{k-1}d_{k-1} + \omega_{k-1},~\xi_0\sim \Xi_0,\label{subeq:obstacle-dynamics}\\
    \zeta_k &= \Phi \xi_k + \nu_k,\label{subeq:obstacle-output}\\
    d_{k-1} &= \pi(\mathcal{I}_{k-1}), \label{subeq:obstacle-policy}
\end{align}
\end{subequations}
where $\xi_k\in \mathbb{R}^{n_\xi}$ and $d_k\in\mathbb{R}^{n_d}$ are obstacle state and input, and $\zeta_k\in \mathbb{R}^{n_\zeta}$ is the output measured by the ego agent. The vector field $g$ is assumed to be continuously differentiable. The input and output matrices $B_{k-1}$ and $\Phi$ are assumed to satisfy that $\Phi B_{k-1}$ has full column rank. The policy $\pi$ is designed by its user to produce control input based on a set of arguments $\mathcal{I}$, which may involve latent information that is not observable by the ego agent (e.g., goal state).  In addition, $\omega_k\in \mathbb{R}^{n_y}$ and $\nu_k\in\mathbb{R}^{n_z}$ are assumed to be independent and identically distributed zero-mean Gaussian noise, with covariance matrices $Q \triangleq \mathbb{E}[\omega_k \omega_k^\top] \succeq 0$ and $R \triangleq \mathbb{E}[\nu_k\nu_k^\top]\succ 0$. 

While $g,B,\Phi$, and the initial distribution $\Xi_0$ are assumed to be known, the full state vector $\xi$, the input $d$, the true policy $\pi$ and the composition of $\mathcal{I}$ are unknown. An approximate behavior model $\hatpi$ is constructed based on some selected arguments $\hat{\mathcal{I}}$. The choice of the behavior model could range from Constant Turn Rate and Velocity (CTRV), Constant Velocity (CV), to more sophisticated approaches like GPR, Recurrent Neural Networks (RNNs), and transformers, depending on available data and computational resources~\cite{huang2022survey}.

The objective is to develop a safe controller ensuring
\begin{equation}\label{eq:safety-constraint}
     s(x_k,\xi_k)\leq 0~~\forall k = 0,1,\ldots,
\end{equation}
where $s:\mathbb{R}^{n_x} \times \mathbb{R}^{n_\xi}\rightarrow \mathbb{R}$ is a differentiable safety measure.

\subsection{Limitations of a Conventional Approach}
Given the behavior model $\hatpi$,  it is common practice to formulate the following MPC problem:  
\begingroup
\allowdisplaybreaks
\begin{subequations}\label{eq:MPC}
\begin{align}
    \min_{\mathfrak{u}} &\sum_{l=0}^{L-1} r(x_{k+l},u_{k+l})+p(x_{k+L})\\\
    \textbf{s.t}~&x_{k+l+1}=f(x_{k+l}, u_{k+l}),~l=0,\ldots,L-1\\
    &d^\hatpi_{k+l-1} = \hatpi(\hat{\mathcal{I}}_{k+l-1}),~l=1,\ldots,L~\label{subeq:policy}\\
    &\hatxi_{k+l}= g(\hatxi_{k+l-1}) + B_{k-1}d^\hatpi_{k+l-1},~l=1,\ldots,L\label{subeq:MPC-obstacle}\\
    &s(x_{k+l},\hatxi_{k+l})\leq 0,~l=0,\ldots,L~\label{subeq:MPC-safety-constraint}\\
    &u_{k+l}\in \mathbb{U},~l=0,\ldots, L-1,
\end{align}
\end{subequations}    
\endgroup
where $\mathfrak{u}=\{u_{k},\ldots,u_{k+L-1}\}$,  $\hatxi$ is a state estimation computed using classical filtering methods such as Extended Kalman Filter (EKF), and  $d^\hatpi$ is a control input produced using the behavior model $\hatpi$ at each time step. The functions $r:\mathbb{R}^{n_x}\times\mathbb{R}^{n_u}\rightarrow \mathbb{R}$ and $p:\mathbb{R}^{n_x}\rightarrow \mathbb{R}$ define the objective of an optimization. To produce a safe control input for the ego agent, the inputs are optimized over a horizon $L$. The first element $u_k^\star$ of the optimal control sequence $\mathfrak{u}^\star$ is applied and this process is repeated for each time step~\cite{fernandez1995model}. While employing this approach was successful in various scenarios, it inherently suffers from \textbf{two limitations}: 1) biased state estimation and 2) distribution shift in the state prediction. 

Specifically, the input gap, represented by the difference $d^{\hat{\pi}}_{k-1} - d_{k-1}$, introduces a bias in the dynamical system~\eqref{eq:obstacle-dynamics}. This compromises the assumption required by the KF-based methods that their uncertainties are zero-mean Gaussian~\cite{urrea2021kalman}. On top of the this, repeated incorporation of predicted control inputs $d^\hatpi_{k+l-1}$ in each prediction step~\eqref{subeq:MPC-obstacle} compounds inaccuracies in the state prediction distribution. This distribution shift reduces the reliability of the safety constraint~\eqref{subeq:MPC-safety-constraint}, which could lead to undesirable outcomes. Fig.~\ref{fig:Problem} visualizes the biased estimation and compounding effect of inaccurate behavior model in the predictive distribution.
\begin{figure}[H]
    \centering \includegraphics[width=0.9\linewidth]{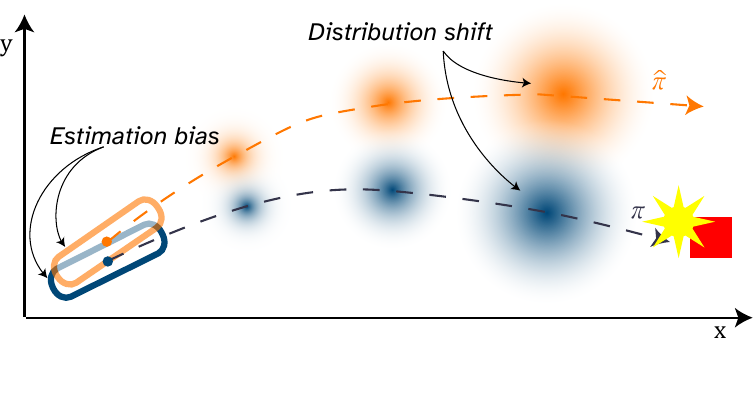}
    \vspace{-20pt}
    \caption{Illustration of estimation bias and compounding predictive distribution shift induced by the behavior gap $\hatpi - \pi$. Even small estimation error and behavior gap of an obstacle can result in large shifts in its predictive distributions. An unsafe trajectory can be predicted as safe.}
    \label{fig:Problem}
\end{figure}

To address the first limitation, we present the SSIE algorithm in Sec.~\ref{sec:ssie} that computes the unbiased estimate of the current state through optimal estimation of the previous input. Then, we formulate a DR-MPC problem in Sec.~\ref{sec:DRMPC} to overcome the second challenge of distribution shift. We evaluate the confidence of the behavior model through the SSIE process and use it to determine the size of the ambiguity set such that the conservatism is controlled according to the local accuracy of the behavior model (Sec.~\ref{subsec:model-confidence}). The integrated method is presented as an \texttt{SIED-MPC} algorithm in Sec.~\ref{subsec:SIED-MPC}. The efficacy of our method is analyzed in Sec.~\ref{sec:simulation}.

\section{Simultaneous State and Input Estimation}\label{sec:ssie}

In this section, we propose the SSIE algorithm, a modification of the ISE algorithm~\cite{yong2016unified,fang2017ensemble}, to compute the optimal estimates of the state $\xi$ and the input gap $d^\hatpi - d$ at each time step along with their covariance matrices when the output $\zeta$ is measured. This process corrects the biased state prediction induced by the inaccurate behavior model. The resulting mean state estimation is proven to be unbiased.

Denote the estimates of the true state $\xi_k$ at each time step $k$ by $\hatxi_k$, where $\hatxi_0 = \mathbb{E}_{\xi_0\sim\Xi_0}[\xi_0]$. For $k=1,2\ldots$, we linearize~\eqref{subeq:obstacle-dynamics} around $\hatxi_{k-1}$ to obtain 
\begin{align}
    \xi_k= g(\hat\xi_{k-1}) + A_{k-1}\tilde{\xi}_{k-1} + B_{k-1}d_{k-1} + \omega_{k-1} + e_{k-1},\label{eq:linearization}
\end{align}
where $\tilde{\xi}_{k-1}\triangleq \xi_{k-1}-\hatxi_{k-1}$ and $A_{k-1} \triangleq \frac{\partial g(\xi)}{\partial \xi}\big \vert_{\xi=\hatxi_{k-1}}$. Higher order errors are lumped in $e_{k-1}$, which we assume to be negligible in the remainder of this work. Furthermore, we assume that $\|A_{k-1}\|\triangleq \sup_{\xi \in \mathbb{R}^{n_\xi}/\{\mathbf{0}\}}\frac{\|A_{k-1}\xi\|_2}{\|\xi\|_2}$ is bounded. 

\subsection{Input Gap Estimation}\label{subsec:input-estimate}
The true input gap $\Delta_{k-1} \triangleq d_{k-1}- d^\hatpi_{k-1}$ is estimated by $\hat{\Delta}_{k-1}$, which is modeled as 
\begin{equation}\label{eq:dhat-model}
\hat{\Delta}_{k-1} = M_k(\zeta_{k}-\hat{\zeta}_{k|k-1}), 
\end{equation}
where the \emph{temporary} output prediction is produced using the behavior model $\hatpi$ as $$\hat{\zeta}_{k|k-1} = \Phi g(\hatxi_{k-1}) + B_{k-1}d^\hatpi_{k-1},~d^\hatpi_{k-1} = \hatpi(\hat{\mathcal{I}}_{k-1}).$$ 
Expanding~\eqref{eq:dhat-model} using~\eqref{eq:linearization} and rearranging, we arrive at
\begin{equation}\label{eq:dhat-rearranged}
\begin{aligned}
    \hat{\Delta}_{k-1} - M_k\Phi B_{k-1} \Delta_{k-1}= M_k \breve{p}_k,
\end{aligned}
\end{equation}
where $\breve{p}_k\triangleq \Phi A_{k-1}\tilde{\xi}_{k-1} + \Phi \omega_{k-1} + \nu_{k}.$ Obtaining $M_k$ that gives the least mean squared error (MSE) is equivalent to solving the following constrained optimization problem: 
\begin{equation}\label{eq:minimization-mk}
    \min_{M_k} \text{cov} (M_k\breve{p}_k)\quad s.t.~~ M_k\Phi B_{k-1} = \mathbb{I}.
\end{equation}
\begin{proposition}\label{proposition:GMT}
    Suppose $\mathbb{E}_{\xi_{k-1} \sim \Xi_{k-1}}[\tilde{\xi}_{k-1}]=0$ holds. Then, the solution to the constrained optimization problem~\eqref{eq:minimization-mk} is given by\footnote{ The assumptions $R\succ 0$, $Q\succeq 0$, $\hat{\Sigma}^\xi_{k-1}\succeq 0$ and the full column rank of $\Phi B_{k-1}$ are used to ensure the existence of inverses in~\eqref{eq:Gauss-Markov-Theorem-solution}.}
    \begin{equation}\label{eq:Gauss-Markov-Theorem-solution}
    M_k = (B_{k-1}^\top \Phi ^\top \breve{P}_k^{-1}\Phi B_{k-1})^{-1}B_{k-1}^\top \Phi ^\top \breve{P}_k^{-1},
    \end{equation}
    where $\breve{P}_{k} = \Phi A_{k-1} \Sigma_{k-1}^\xi A_{k-1}^\top \Phi^\top  + \Phi Q \Phi ^\top + R$, $\Sigma_{k-1}^\xi \triangleq \mathbb{E}[\tilde{\xi}_{k-1} \tilde{\xi}_{k-1}^\top]$ is the covariance matrix of the state estimate and the initial value $\Sigma_0^\xi$ is given by $\cov(\xi_0)$.
\end{proposition}
\begin{proof}
We can compute the covariance of $M_k\breve{p}_k$ as
\begingroup
\allowdisplaybreaks
\begin{equation*}
\begin{aligned}
    \cov &(M_k\breve{p}_k) \\
    &= M_k (\Phi A_{k-1} \Sigma_{k-1}^\xi A_{k-1}^\top \Phi^\top  + \Phi Q_{k-1} \Phi^\top + R)M_k^\top\\
    &= M_k \breve{P}_k M_k^\top.
\end{aligned}
\end{equation*}
\endgroup
Subsequently,~\eqref{eq:minimization-mk} is equivalently expressed as
\begin{equation*}
    \min_{M_k} M_k \breve{P}_k M_k^\top \quad s.t. ~~ M_k\Phi B_{k-1} = \mathbb{I},
\end{equation*}
the solution of which is given by the Gauss-Markov-Theorem (GMT) as~\eqref{eq:Gauss-Markov-Theorem-solution}~\cite[Sec.~3]{kailath2000linear}.
\end{proof}

% \begin{remark} From the construction, we have $R\succ 0$, $Q\succeq 0$, $\hat{\Sigma}^\xi_{k-1}\succeq 0$, and $\Phi B_{k-1}$ has full column rank. These conditions ensure the invertibility of $\breve{P}_k$ and $B_{k-1}^\top\Phi^\top \breve{P}_k^{-1}\Phi B_{k-1}$.
% \end{remark}

\subsection{State Prediction}\label{subsec:state-prediction}
Using the optimal input gap estimate $\hat{\Delta}_{k-1}$ derived from~\eqref{eq:dhat-model}, we produce the state prediction of $\xi_k$ as
\begin{equation}\label{eq:state-prediction}
    \hatxi_{k|k-1} = g(\hatxi_{k-1})+B_{k-1}(d^\hatpi_{k-1}+\hat{\Delta}_{k-1}).
\end{equation}
To analyze the accuracy of this prediction, we use the linearized dynamics~\eqref{eq:linearization} to obtain
\begin{equation}\label{eq:xi-tilde-prediction}
\begin{aligned}
    \tilde{\xi}_{k|k-1}&\triangleq \xi_{k}-\hat{\xi}_{k|k-1} \\
    &= (\mathbb{I}-B_{k-1}M_k\Phi)A_{k-1}\tilde{\xi}_{k-1}+ \\
    &~~~~(\mathbb{I}-B_{k-1}M_k\Phi)\omega_{k-1} -B_{k-1}M_k\nu_{k}.
\end{aligned}     
\end{equation}
Notice that if $\mathbb{E}_{\xi_{k-1}\sim \Xi_{k-1}}[\tilde{\xi}_{k-1}] = 0$, we immediately have \begin{equation}\label{eq:xi-tilde-prediction-error}
    \mathbb{E}\left[\tilde{\xi}_{k|k-1}\right]= 0.
\end{equation}
This implies a crucial property of the SSIE process that the corrected mean state prediction is \emph{unbiased}. 

Likewise, the covariance matrix of the state prediction error is obtained as
\begin{align}
&\hat{\Sigma}_{k|k-1}^\xi \triangleq \mathbb{E}\left[\tilde{\xi}_{k|k-1}\tilde{\xi}_{k|k-1}^\top\right]\notag \\
&\stackrel{\eqref{eq:xi-tilde-prediction}}{=} (\mathbb{I}-B_{k-1}M_k\Phi)A_{k-1}~\hat{\Sigma}_{k-1}^\xi A_{k-1}^\top(\mathbb{I}-B_{k-1}M_k\Phi)^\top \notag\\
&~~+ (\mathbb{I}-B_{k-1}M_k\Phi)Q(\mathbb{I}-B_{k-1}M_k\Phi)^\top \label{eq:cov-xi-prior} \\
&~~+ B_{k-1}M_kRM_k^\top B_{k-1}^\top, \notag
\end{align}
where we used the independence of the random variables.

\subsection{State Estimation}\label{subsec:state-estimation}
The estimation of the current state $\hat{\xi}_k$ is modeled as
\begin{equation}\label{eq:state-estimation}
    \hat{\xi}_{k} = \hat{\xi}_{k|k-1} + L_k(\zeta_{k}-\Phi \hat{\xi}_{k|k-1}),
\end{equation}
where $L_k\in \mathbb{R}^{n_\xi\times n_\zeta}$ needs to be chosen to minimize the covariance matrix $\Sigma_{k}^\xi$. Plugging~\eqref{subeq:obstacle-output} into~\eqref{eq:state-estimation}, we obtain $\hat{\xi}_{k} = \hat{\xi}_{k|k-1} + L_k(\Phi \xi_{k} + \nu_{k}-\Phi \hat{\xi}_{k|k-1})$, which leads to 
\begin{equation}\label{eq:posterior-estimate}
    \tilde{\xi}_k = (\mathbb{I}-L_k\Phi)\tilde{\xi}_{k|k-1} - L_k\nu_k.
\end{equation}
Taking the expectation on both sides and applying~\eqref{eq:xi-tilde-prediction} gives
\begin{align*}
    &\mathbb{E}_{\xi_k\sim \Xi_k}\left[\tilde{\xi}_{k}\right] = \\
    &(\mathbb{I}-L_k\Phi)(\mathbb{I}-B_{k-1}M_k\Phi)A_{k-1}\mathbb{E}_{\xi_{k-1}\sim \Xi_{k-1}}[\tilde{\xi}_{k-1}] = 0, \notag
\end{align*}
implying \emph{unbiased} state estimation. 
\begin{proposition}
    Let the prior covariance matrix $\Sigma_{k|k-1}^{\xi}$ be defined as in~\eqref{eq:cov-xi-prior}. Then, the posterior covariance matrix $\Sigma_{k}^{\xi}\triangleq \mathbb{E}[\tilde{\xi}_k\tilde{\xi}_k^\top]$ is obtained as
\begin{align}
   &\hat{\Sigma}_{k}^{\xi}=(\mathbb{I}-L_k\Phi)\hat{\Sigma}_{k|k-1}^\xi(\mathbb{I}-L_k\Phi)^\top + L_kRL_k^\top+ \label{eq:xi-covariance-estimate}\\
    &(\mathbb{I}-L_k\Phi)B_{k-1}M_kRL_k^\top + L_kRM_k^\top B_{k-1}(\mathbb{I}-L_k\Phi)^\top.\notag
\end{align}
\end{proposition}
\begin{proof}
    Evaluating $\hat{\Sigma}^\xi_k$ using~\eqref{eq:posterior-estimate}, we get
    \begin{align}
    &\hat{\Sigma}_{k}^{\xi}=(\mathbb{I}-L_k\Phi)\hat{\Sigma}_{k|k-1}^\xi(\mathbb{I}-L_k\Phi)^\top + L_kRL_k^\top+ \label{xitilde-covariance}\\
    &(\mathbb{I}-L_k\Phi)\mathbb{E}[\tilde{\xi}_{k|k-1}\nu_k^\top]L_k^\top + L_k\mathbb{E}[\nu_k\tilde{\xi}_{k|k-1}^\top](\mathbb{I}-L_k\Phi)^\top,\notag
    \end{align}
    where $\tilde{\xi}_{k|k-1}$ is given in~\eqref{eq:xi-tilde-prediction}. Moreover, we have
    \begin{equation*}
        \mathbb{E}[\tilde{\xi}_{k|k-1}\nu_k^\top] = B_{k-1}M_k\mathbb{E}[\nu_k\nu_k^\top] = B_{k-1}M_kR,
    \end{equation*}
    where we used the fact that $\tilde{\xi}_{k-1}$ and $\omega_{k-1}$ are uncorrelated with $\nu_k$. Substituting this result (and a similar result for its transpose) into~\eqref{xitilde-covariance} gives~\eqref{eq:xi-covariance-estimate}, concluding the proof.
\end{proof}
Next, we optimize $L_k$ as the minimizer of $\tr(\hat{\Sigma}_{k}^\xi)$.
\begin{proposition}
    The solution to the unconstrained optimization
    \begin{equation*}
    \min_{L_k}\tr(\hat{\Sigma}_{k}^\xi)
    \end{equation*}
    for $\hat{\Sigma}_{k}^\xi$, defined in~\eqref{eq:xi-covariance-estimate}, is given by
    \begin{align*}
        &L_k = (\hat{\Sigma}_{k|k-1}^\xi \Phi^\top-B_{k-1}M_kR)\\
        &~~~\cdot(R+\Phi \hat{\Sigma}^\xi_{k|k-1}\Phi^\top -\Phi B_{k-1}M_k R -RM^\top_k B_{k-1}^\top \Phi^\top)^{-1}.\notag
    \end{align*}
\end{proposition}
\begin{proof}
Note from~\eqref{eq:xi-covariance-estimate} that $\hat{\Sigma}^\xi_k$ is a quadratic (convex) function of $L_k$. Hence, the global solution to the minimization problem can be obtained by solving the following first-order optimality condition: 
    \begin{equation*}
    \begin{aligned}
        &\frac{\partial~\tr(\hat{\Sigma}_{k}^\xi)}{\partial L_k} = -2\bigg(\hat{\Sigma}_{k|k-1}^\xi\Phi^\top - L_k\Phi \hat{\Sigma}_{k|k-1}^\xi \Phi^\top - L_kR\\
        &- B_{k-1}M_kR + L_k\Phi B_{k-1}M_kR + L_kRM_k^\top B_{k-1}^\top \Phi^\top\bigg)=0.
    \end{aligned}
    \end{equation*}
The result follows by solving the above equality for $L_k$.
\end{proof}
%%%%%%%%%%%%%%%%%%%%%%%%%%%%%
%%%%% SSIE Algorithm %%%%%%%
%%%%%%%%%%%%%%%%%%%%%%%%%%%%
\begin{algorithm}
\setstretch{1}

\caption{\texttt{SSIE}}\label{alg:SSIE}
\SetKwComment{Comment}{$\vartriangleright$ }{}
\SetKwInput{KwGiven}{Given}
\SetKwInput{KwInitialize}{Initialize}

\KwGiven{$g, B_{\{1,2,\ldots\}}, \Phi, Q, R, \hatpi, \Xi_0$}
\KwInitialize{$\hatxi_0 = \mathbb{E}_{\xi_0\sim \Xi_0}[\xi_0],\Sigma_{0}^\xi=\cov(\xi_0)$}
\KwIn{$\zeta_{k}$}
\Comment{\colorbox{yellow}{\bf Input gap Estimation}}
\vspace{3pt}

$A_k \gets \frac{\partial g(\xi)}{\partial \xi}\bigg\vert_{\xi =\hat{\xi}_{k-1}}$\;
$\hat{\zeta}_{k|k-1} \gets \Phi g(\hatxi_{k-1}) + \Phi B_{k-1}d^\hatpi_{k-1}$\;
$\breve{P}_{k} \gets \Phi A_{k-1} \Sigma_{k-1}^\xi A_{k-1}^\top \Phi^\top  + \Phi Q \Phi^\top + R$\;
$M_k \gets (B_{k-1}^\top \Phi^\top \breve{P}_k^{-1}\Phi B_{k-1})^{-1}B_{k-1}^\top \Phi^\top \breve{P}_k^{-1}$\;
$\hat{\Delta}_{k-1} \gets M_k(\zeta_{k}-\hat{\zeta}_{k|k-1})$\;
$\Sigma_{k-1}^{\Delta} \gets M_k \breve{P}_k M_k^\top$\;
\vspace{3pt}
\Comment{\colorbox{yellow}{\bf State prediction}}
\vspace{3pt}

$\hatxi_{k|k-1} \gets g(\hatxi_{k-1}) + B_{k-1}(d^\hatpi_{k-1}+\hat{\Delta}_{k-1})$\;
$\hat{\Sigma}_{k|k-1}^\xi \gets (\mathbb{I}-B_{k-1}M_k\Phi)A_{k-1}~\hat{\Sigma}_{k-1}^\xi A_{k-1}^\top(\mathbb{I}-B_{k-1}M_k\Phi)^\top + (\mathbb{I}-B_{k-1}M_k\Phi)Q(\mathbb{I}-B_{k-1}M_k\Phi)^\top + B_{k-1}M_kRM_k^\top B_{k-1}^\top$\;
\vspace{3pt}
\Comment{\colorbox{yellow}{\bf State estimation}}
\vspace{3pt}

$L_k \gets (\Sigma_{k|k-1}^\xi \Phi^\top-B_{k-1}M_kR)(R+\Phi \hat{\Sigma}^\xi_{k|k-1}\Phi^\top -\Phi B_{k-1}M_k R -RM^\top_kB_{k-1}^\top \Phi^\top)^{-1}$\;

$\hatxi_{k} \gets \hatxi_{k|k-1} + L_k(\zeta_{k}-\Phi \hatxi_{k|k-1})$\;

$\hat{\Sigma}_{k}^\xi \gets (\mathbb{I}-L_k\Phi)\hat{\Sigma}_{k|k-1}^\xi(\mathbb{I}-L_k\Phi)^\top + L_kRL_k^\top +(\mathbb{I}-L_k\Phi)B_{k-1}M_kRL_k^\top + L_kRM_k^\top B_{k-1}^\top(\mathbb{I}-L_k\Phi)^\top$\;

\vspace{3pt}
\KwOut{$\hatxi_{k},\hat{\Delta}_{k-1},\hat{\Sigma}_{k}^\xi,\Sigma_{k-1}^\Delta$}
\end{algorithm}

\section{Distributionally Robust Safety Constraint}\label{sec:DRMPC}
In this section, we propose a method to address the distributional shifts that result from an inaccurate behavior model. We begin by explaining the widely recognized CVaR risk measure in Sec.~\ref{subsec:cvar} to transform the deterministic safety constraint~\eqref{subeq:MPC-safety-constraint} into a probabilistic constraint. Subsequently, in Sec.~\ref{subsec:DR-CVaR}, we develop a distributionally robust optimization (DRO) problem that incorporates a confidence-based ambiguity. We also present its tractable reformulation for its practical implementation.

%%%%%%%%%%%%%%%%%%%%%%%%%%%%%%%%%
\subsection{Conditional Value at Risk}\label{subsec:cvar}
Conditional Value at Risk (CVaR) has been widely used in robotics, finance, and other domains for its adequacy in capturing risks in an extreme tail and has a tractable form when used in risk-aware optimization problems~\cite{mohajerin2018data}. CVaR of a random loss $S\sim \mathcal{P}$ measures the expected value of $S$ in the conditional distribution of its upper $(1-\alpha)$ tail for $\alpha \in [0,1)$~\cite{xin_bounds_2012} and is expressed as 
\begin{equation}\label{eq:cvar-definition}
    \cvar^\mathcal{P}_\alpha[S] \triangleq \inf_{z\in \mathbb{R}} \mathbb{E}_{S\sim \mathcal{P}}\left[z+\frac{\max \{S-z,0\}}{1-\alpha}\right].
\end{equation}

We aim to convert deterministic safety constraint~\eqref{subeq:MPC-safety-constraint} into a relaxed probabilistic constraint using CVaR to accommodate the stochastic characteristics of the obstacle trajectory predictions~\eqref{eq:obstacle-dynamics}. To this end, plug~\eqref{subeq:MPC-safety-constraint} into~\eqref{eq:cvar-definition} to obtain $$\cvar^{\hatp_{k+l|k}}_\alpha[s(x_{k+l},\hatxi_{k+l})]\leq 0,~l=1,2,\ldots,$$ where $\hatp_{k+l|k}$ is the nominal distribution of the loss function $s(x_{k+l},\hatxi_{k+l})$. In the following, we explain the two steps to derive $\hatp_{k+l|k}$: 1) we propagate multi-step prediction of the nominal obstacle state distribution and 2) compute the approximation of the loss distribution. 
%%%%%%%%%%%%%%%%%%%%%%%%%%%%%%
\subsubsection{Propagating Nominal State Distribution}\label{subsec:state-predicting}

The mean state predictions, denoted $\muxi_{k+l}$ for $l=0, 1,\ldots$, are initialized for $l=0$ as $\muxi_{k} = \hatxi_{k}$, and are derived at each $k$ by the SSIE procedure. By autoregressively applying the given policy model $\hatpi$, we obtain the mean state predictions as
\begin{equation}\label{eq:multistep-mean-prediction}
\muxi_{k+l|k} = g(\muxi_{k+l-1|k}) + B_{k+l-1}d^\hatpi_{k+l-1}, \quad l = 1, 2, \ldots.
\end{equation}
Likewise, we propagate the covariance matrix as
\begin{align}\label{eq:multistep-cov-prediction}
\hat{\Sigma}^\xi_{k+l|k} = A_{k+l}\hat{\Sigma}^\xi_{k+l-1|k}A_{k+l}^\top + Q,
\end{align}
where we assumed the invariance of the state prediction error~\eqref{eq:xi-tilde-prediction-error} to hold throughout the multi-step predictions. 

%%%%%%%%%%%%%%%%%%%%%%%%%%%%%%%%%%%%%%%%%%%%%%%%%%
\subsubsection{Computing Nominal Loss Distribution}
The Gaussian approximation of the nominal loss distribution $\hatp_{k+l|k}$ is obtained by the first-order Taylor series expansion of the loss function around the nominal state distribution as $\hatp_{k+l|k} = \mathcal{N}(\muhat^s_{k+l},\sigmahat^s_{k+l})$, where
\begingroup
\allowdisplaybreaks
\begin{subequations}
\begin{align}
    \muhat^s_{k+l} &= s(x_{k+l},\muxi_{k+l}) \label{subeq:mu_s}\\
    \sigmahat^s_{k+l} &= \nabla s(x_{k+l},\muxi_{k+l})^\top \hat{\Sigma}^\xi_{k+l} \nabla s(x_{k+l},\muxi_{k+l}).\label{subeq:sigma_s}
\end{align}    
\end{subequations}
\endgroup
The derivation follows the similar spirit to Section~\ref{sec:ssie} (or see~\cite[Section~4.3]{benaroya2005probability}). 
%%%%%%%%%%%%%%%%%%%%%%%%%%%%%%%%%%%%%
%%%%%%%%%%%%%%%%%%%%%%%%%%%%%%%%%%%%%
\subsection{Distributionally Robust CVaR}\label{subsec:DR-CVaR}
The nominal state distribution is subject to compounding error induced by the input gap at each time step. To address this problem, we consider a set of distributions $\mathbb{P}_{k+l|k}$ around $\hatp_{k+l|k}$, within which the worst-case safety loss is searched for. Correspondingly, the DR-CVaR measure is defined as
\begin{equation}\label{eq:dr-cvar}
\begin{aligned}
    &\drcvar^{\hat{{\mathcal{P}}}_{k+l|k}}_{\alpha,\theta}\left[s(x_{k+l},\hatxi_{k+l|k})\right]\\
    &~~~\triangleq \sup_{\hat{\mathcal{P}}\in \mathbb{P}_{k+l|k}}\cvar^{\hat{\mathcal{P}}}_{\alpha}\left[s(x_{k+l},\hatxi_{k+l|k})\right].
\end{aligned}   
\end{equation}
Two questions arise: 1) How can we choose the \emph{ambiguity set} $\mathbb{P}_{k+l|k}$? 2) Given the search space $\mathbb{P}_{k+l|k}$, how can we solve for the outer supremum? To answer these questions, we introduce the Wasserstein ambiguity set~\cite{rahimian2019distributionally} and propose a systematic method to choose the \emph{radius} of an ambiguity set. Then, we upper bound the constraint $\drcvar$ for computational tractability in solving the DRO problem.
%%%%%%%%%%%%%%%%%%%%%%%%%%%%%%%%%
\begin{definition}A Wassersteien ambiguity set of a ball of radius $\theta>0$ centered at $\hatp_{k+l|k}$ is defined as
\begin{equation*}
\mathbb{P}_{k+l|k} \triangleq \{\hat{\mathcal{P}}\in \mathfrak{F}(\mathbb{R})|W_2(\hatp,\hatp_{k+l|k})\leq \theta\},    
\end{equation*}
where $\mathfrak{F}(\mathbb{R})$ is the family of all probability distributions with finite second moments supported on $\mathbb{R}$. Moreover, for some distributions $\mathcal{R}_{\{1,2\}} \in \mathfrak{F}(\mathbb{R})$, $W_2(\mathcal{R}_1,\mathcal{R}_2)$  is the 2-Wasserstein distance between the two distributions: 
\begin{equation*}
\begin{aligned}
    &W_2(\mathcal{R}_1,\mathcal{R}_2) \triangleq\\
    &\left(\inf_{\psi \in \Psi(\mathcal{R}_1,\mathcal{R}_2)} \int_{\mathbb{R}\times \mathbb{R}} \|r_1-r_2\|^2_2~\psi(dr_1,dr_2)\right)^{\frac{1}{2}},
\end{aligned}
\end{equation*}
where $\Psi(\mathcal{R}_1,\mathcal{R}_2)$ denotes the set of all joint probability distributions of $r_1\in \mathbb{R}$ and $r_2\in \mathbb{R}$ with marginals $\mathcal{R}_1$ and $\mathcal{R}_2$, respectively~\cite[Definition~1]{kuhn_wasserstein_2019}.
\end{definition}

\subsubsection{Confidence-Based Ambiguity Set}\label{subsec:model-confidence}
A commonly used approach to define $\theta$ is to heuristically set $\theta$ sufficiently large to cover the maximum plausible distribution shifts~\cite{hakobyan2021wasserstein}. However, such approach lacks flexibility and reduces the feasible solution set, often encountering computational challenges~\cite{ben2010soft}. Inspired by~\cite{de2000mahalanobis}, we use the input gap and its covariance estimate $\hat{\Delta}_k$ and $\Sigma^{\Delta}_k$ to systematically choose $\theta$.

Specifically, to evaluate the confidence of the behavior model, denoted as $F^\hatpi_k$, we cache the last $m$-step estimates $\{(\hat{\Delta}_{k-m},\Sigma^\Delta_{k-m}),\ldots,(\hat{\Delta}_{k-1},\Sigma^\Delta_{k-1})\}$ to compute
\begin{equation}\label{eq:mahalanobis}
    F^\hatpi_k \triangleq \sqrt{\frac{1}{m}\sum_{l=1}^m\hat{\Delta}^\top_{k-l} (\Sigma^{\Delta}_{k-l})^{-1} \hat{\Delta}_{k-l}}.
\end{equation}
A smaller $F^\hatpi_k$ indicates that the input gap distribution is reliably close to zero.

We propose the following dynamic ambiguity set model that controls the conservatism in the DRO problem based on the accuracy of the behavior model: 
\begin{equation}\label{eq:dynamic-ambiguity-set}
    \mathbb{P}_{k+l|k} \triangleq \{\hat{\mathcal{P}}\in \mathfrak{F}(\mathbb{R})|W_2(\hatp,\hatp_{k+l|k})\leq \theta(F^\hatpi_k)\},
\end{equation}
where $\theta(F^\hatpi_k) \triangleq \theta_{\max}\tanh(\tau F^\hatpi_k)$.
In this model, $\theta_{\max} \in \mathbb{R}_{\geq 0}$ represents the maximum radius, and $\tau \in \mathbb{R}_{\geq 0}$ is a temperature parameter that modulates the sensitivity of the ambiguity radius to confidence $F^\hatpi_k$. 

%%%%%%%%%%%%%%%%%%%%%%%%%%%%%%%%%%%%%%%%
\subsubsection{Tractable Reformulation}
Evaluating the safety constraint~\eqref{eq:dr-cvar} with the dynamic ambiguity set~\eqref{eq:dynamic-ambiguity-set} presents a computational challenge as it involves solving an infinite dimensional optimization problem. To address this, we turn to the following Lemma:
\begin{lemma}(Gelbrich Bound~\cite[Theorem~2.1]{gelbrich_formula_1990})\label{lemma:gelbrich-bound}
    Let two distributions $\hatp$ and $\hat{\mathcal{P}}_{k+l|k}$ on $\mathbb{R}$ have means $\muhat^s,\muhat^s_{k+l}\in \mathbb{R}$ and variances $\sigmahat^s,\sigmahat^s_{k+l} \in \mathbb{R}_{\geq 0}$, respectively. Then, the following inequality holds:
    \begin{equation*}
        W_2(\hatp,\hatp_{k+l|k}) \geq \sqrt{(\muhat^s-\muhat^s_{k+l})^2 + (\sigmahat^s-\sigmahat^s_{k+l})^2}.
    \end{equation*}
    The bound is exact if $\hatp$ and $\hatp_{k+l|k}$ are elliptical distributions with the same density generator.
\end{lemma}

\begin{theorem}
Suppose that the nominal distribution of the loss $s(x_{k+l},\hatxi_{k+1|k})$ is given as $\mathcal{N}(\mu^s_{k+l}, (\sigma^s_{k+l})^2)$, where $\mu^s_{k+l}$ and $\sigma^s_{k+l}$ are defined in~\eqref{subeq:mu_s} and~\eqref{subeq:sigma_s}, respectively. Then, we have the following upper-bound for the $\drcvar$ with the radius $\theta(F^\hatpi_k)$:
\begin{align}
&\drcvar_{\alpha,\theta(F^\hatpi_k)}^{\hatp_{k+l|k}} [s(x_{k+l},\hatxi_{k+l})] \notag\\
&~~~~~~~~~\leq \mu^s_{k+l}+\gamma \sigma^s_{k+l} + \theta (F^\hatpi_k)\sqrt{1+\gamma^2},\label{eq:DR-upperbound}
\end{align}
where $\gamma \triangleq \sqrt{\alpha/(1-\alpha)}$.
\end{theorem}
\begin{proof}
    Let $\bar{\mathbb{P}}_{k+l|k}^\theta$ denote an ambiguity set around the estimated mean $\muhat^s_{k+l}$ and variance $(\sigmahat^s_{k+l})^2$, i.e. $ \bar{\mathbb{P}}_{k+l|k}^\theta \triangleq \{(\mu,\sigma^2) \in \mathbb{R}\times \mathbb{R}_{\geq 0}~|~(\mu-\muhat^s_{k+l})^2 + (\sigma-\sigmahat^s_{k+l})^2 \leq \theta^2 \}.$
    It follows from Lemma~\ref{lemma:gelbrich-bound} that $\mathbb{P}_{k+l|k}\subseteq \barp$~\cite[Proposition~1]{kuhn_wasserstein_2019}, which gives
    \begin{align*}
        &\drcvar_{\alpha,\theta(F^\hatpi_k)}^{\hatp_{k+l|k}} [s(x_{k+l},\hatxi_{k+l})] \\
        &\leq \sup_{\bar{\mathcal{P}}\in \barp}\cvar^{\bar{\mathcal{P}}}_{\alpha}[s(x_{k+l},\hatxi_{k+l|k})]\\
        &= \sup_{(\mu_{k+l},\sigma^2_{k+l})\in \barp} \sup_{\mathcal{P}\in \mathcal{C}(\mathbb{R},\mu_{k+l},\sigma_{k+l}^2)} \cvar^{\mathcal{P}}_{\alpha}[s(x_{k+l},\hatxi_{k+l})],
    \end{align*}
where $\mathcal{C}(\mathbb{R},\mu_{k+l},\sigma_{k+l})$ denotes the Chebyshev uncertainty set that contains all distributions on $\mathbb{R}$ with mean $\mu_{k+l}$ and variance $\sigma_{k+l}$~\cite{kuhn_wasserstein_2019}. 

The inner supremum has a closed form solution given by
$$\sup_{\mathcal{Q}\in \mathcal{C}(\mathbb{R},\mu_{k+l},\sigma_{k+l}^2)} \cvar^{\mathcal{Q}}_{\alpha}[s(x_{k+l},\hatxi_{k+l})] = \mu_{k+l}+\gamma \sigma_{k+l},$$
where $\gamma \triangleq \sqrt{\frac{\alpha}{1-\alpha}}$~\cite[Proposition~2]{yu_general_2009}. For a compact search space (finite second moment), we can rewrite the outer supremum as the following optimization problem:
\begin{align*}
    &\min_{(\mu_{k+l},\sigma_{k+l})} -\mu_{k+l} -\gamma \sigma_{k+l}\\
    \textbf{s.t.}~&  (\mu-\muhat^s_{k+l})^2 + (\sigma-\sigmahat^s_{k+l})^2 \leq \theta^2.
\end{align*}

The dual problem is obtained as
\begin{align}
    &\max_{\lambda\geq 0} \min_{\mu,\sigma} \mathcal{L}(\mu_{k+l},\sigma_{k+l},\lambda)\notag\\
    &=\max_{\lambda\geq 0} ~(\muhat^s_{k+l}+\frac{1}{2\lambda})+\gamma(\sigmahat^s_{k+l}+\frac{\gamma}{2\lambda}) - \lambda \theta^2,\label{eq:dual}
\end{align}
where $\mathcal{L}$ is the Lagrangian and $\lambda$ is its associated multiplier. The solution to this problem is obtained as $\lambda = \frac{1}{\theta} \sqrt{\frac{1+\gamma^2}{2}}$, substituting which into~\eqref{eq:dual} gives the desired result~\eqref{eq:DR-upperbound}.
\end{proof}
%%%%%%%%%%%%%%%%%%%%%%%%%%%%%%%%%%%
%%%%%%%%%%%%%%%%%%%%%%%%%%%%%%%%
\subsection{Distributionally Robust MPC}\label{subsec:SIED-MPC}
Integrating the SSIE algorithm and the DR safety constraint~\eqref{eq:DR-upperbound}, the MPC problem~\eqref{eq:MPC} can be reformulated as 
\begin{align}\label{eq:DRMPC}
    \min_{\bf u,x} &\sum_{l=0}^{L-1} r(x_{k+l},u_{k+l})+p(x_{k+L})\\
    \textbf{s.t}~&x_{k+l+1}=f(x_{k+l}, u_{k+l}),~l=0,\ldots,L-1\notag\\
    &\hatxi_{k+l|k}= g(\hatxi_{k+l-1|k}) + B_{k-1}d^\hatpi_{k+l-1} \notag,~l=1,\ldots,L\\
    &\mu^s_{k+l}+\gamma \sigma^s_{k+l} + \theta (F^\hatpi_k)\sqrt{1+\gamma^2}\leq 0,~l=0,\ldots,L \notag\\
    &u_{k+l}\in \mathbb{U},~l=0,\ldots, L-1\notag.
\end{align}
The reformulated problem~\eqref{eq:DRMPC} is tractable and can be solved by using off-the-shelf nonlinear optimization algorithms~\cite{FORCESNLP}. The \texttt{SIED-MPC} algorithm is summarized in Algorithm~\ref{alg:DR-Control}. 
\begingroup
\allowdisplaybreaks
\begin{algorithm}
\setstretch{1}
\SetKwInput{KwGiven}{Given}
\caption{\texttt{SIED-MPC}}\label{alg:DR-Control}
\SetKwComment{Comment}{$\vartriangleright$ }{}

\KwGiven{$\mathbb{F}=\phi$, $m$, $\theta_{\max}$, $\tau$, $L$}
\For{$t=1:T$}{
Observe $x_t$ and $\zeta_t$\;
$\hatxi_t,\hat{\Delta}_{t-1},\hat{\Sigma}^\xi_{t},\Sigma^\Delta_{t-1} \leftarrow \texttt{SSIE}(\zeta_t)$ \;
$\mathbb{F} \leftarrow \mathbb{F} \cup (\hat{\Delta}_{t-1}, \Sigma^\Delta_{t-1})$\;
\If {$|\mathbb{F}| > m$}{Keep the most recent $m$ pairs}
Evaluate $F^\hatpi_t$ with $\mathbb{F}$ using~\eqref{eq:mahalanobis}\;
$\theta \leftarrow \theta_{\max}\tanh (\tau F^\hatpi_t)$\;
\For{l=1:L}{
Compute $(\muhat^\xi_{t+l}, \hat{\Sigma}^\xi_{t+l})$ using~\eqref{eq:multistep-mean-prediction},\eqref{eq:multistep-cov-prediction}\;
}
Solve nonlinear MPC~\eqref{eq:DRMPC} to obtain $u^\star_t$\;
Apply $u^\star_t$ to the system\;
}
\end{algorithm}
\endgroup

%%%%%%%%%%%%%%%%%%%%%%%%%%%%%%%%%%%%%%%
\section{Simulations}\label{sec:simulation}
To validate the efficacy of the proposed \texttt{SIED-MPC} algorithm, we conducted autonomous driving simulations using the open-source platform CARLA~\cite{Dosovitskiy17}. The objective is to control an ego vehicle to follow a given trajectory $x^g_t$ while avoiding collisions with an obstacle, the behavior model of which is inaccurately known to the ego vehicle. The simulation scenario is illustrated in Fig.~\ref{fig:mainfigure}. Videos of the simulations can be found at~\href{http://sied-mpc.notion.site}{http://sied-mpc.notion.site}.
\begin{figure}[H]
    \centering
    \includegraphics[width = 0.75\columnwidth]{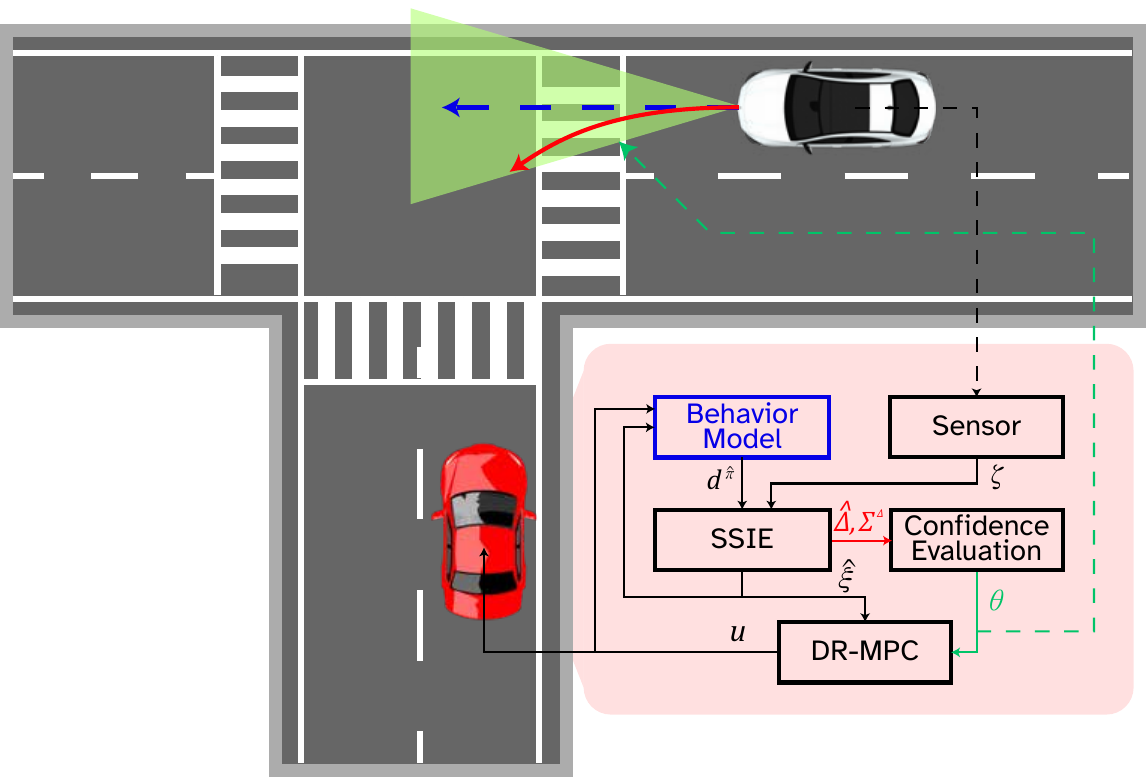}
    \caption{(Scenario description) The goal of the ego vehicle (red) is to safely turn left at the intersection without colliding with the obstacle vehicle (white). While the behavior model predicts the obstacle to adhere to the CSAV model, the actual obstacle unexpectedly steers to the left.}
    \label{fig:mainfigure}
\end{figure}
\subsection{Simulation Setup}
\subsubsection{Dynamics}
The dynamics of the ego vehicle are modeled using the kinematic bicycle model\footnote{Superscript $e$ denotes the \emph{ego} vehicle. We use $o$ to denote the \emph{obstacle}.}:
\begin{align}\label{eq:kinematic-bicycle}
    \begin{bmatrix}
        \mathrm{x}^e_{k+1}\\
        \mathrm{y}^e_{k+1}\\
        \phi^e_{k+1}\\
        v^e_{k+1}
    \end{bmatrix}
    =
    \begin{bmatrix}
        \mathrm{x}^e_k\\
        \mathrm{y}^e_k\\
        \phi^e_k\\
        v^e_k
    \end{bmatrix}
    +T_s
    \begin{bmatrix}
         v^e_k\cos (\phi^e_k+ \beta^e_k)\\
        v^e_k\sin (\phi^e_k + \beta^e_k)\\
        \frac{v^e_k}{L^e_c}\sin(\beta^e_k)\\
        a^e_k
    \end{bmatrix},
\end{align}
where the step size is $T_s = 0.1$, and the car length is $L^e_c= L^o_c = 4.611m$. The state vector $[\mathrm{x}^e_k, \mathrm{y}^e_k, \phi^e_k, v^e_k]^\top$ represents the position, heading angle, and velocity of the vehicle. The control input for the ego vehicle, $u_k = [a^e_k,\delta^e_k]^\top$, comprises the acceleration and the steering angle. Moreover, $\beta^e_k \triangleq \tan^{-1}\left(\frac{1}{2}\tan(\delta^e_k)\right)$ is the slip angle. 

While the same kinematic bicycle model~\eqref{eq:kinematic-bicycle} also represents the dynamics of the obstacle, implementing the \texttt{SIED-MPC} Algorithm~\ref{alg:DR-Control} requires linearizations. To this end, we treat $a$ and $\beta$ as control inputs for the obstacle and the Jacobian matrices $A_{k-1}$ and $B_{k-1}$ are obtained in~\eqref{eq:Jacobian}. Moreover, $\Phi = \mathbb{I}$ and the noise parameters are given as $Q = R = \text{diag}(1,1,0.05,0.05)$.

\begin{figure*}
\raggedleft
\begin{align}
A_{k-1} = \begin{bmatrix}
    1&0&-T_s v^o_{k-1} \sin(\phi^o_{k-1}+ \beta^o_{k-1})& T_s\cos(\phi^o_{k-1}+\beta^o_{k-1})\\
    0&1&T_s v^o_{k-1} \cos(\phi^o_{k-1}+\beta^o_{k-1})& T_s \sin(\phi^o_{k-1}+\beta^o_{k-1})\\
    0&0&1&T_s\sin(\beta^o_{k-1})/L^o_c\\
    0&0&0&1
\end{bmatrix},~
B_{k-1} = \begin{bmatrix}
        0 & -v^o_{k-1}\sin(\phi^o_{k-1}+\beta^o_{k-1}) \\
        0 & v^o_{k-1}\cos(\phi^o_{k-1}+\beta^o_{k-1}) \\
        0 & \frac{T_s v^o_{k-1}\cos (\beta^o_{k-1})}{L^o_c}\\
        T_s & 0
    \end{bmatrix}
\label{eq:Jacobian}
\end{align}
\vspace{-10pt}
\end{figure*}
\subsubsection{MPC}
The cost function for the MPC is defined as 
$r(x_{k+l},u_{k+l}) = \|x_{k+l}-x^g_l\|^2_S + \|\Delta u_k\|^2_T$ and $p(x_{k+L}) = \|x_{k+L} -x^g_L\|^2,$ where $S=\text{diag}(1,1,10,0.2)$ and $T=\text{diag}(0.2,4)$. The constraints are specified by $\mathbb{U} = \{(a_k, \delta_k)\in \mathbb{R} \times \mathbb{R}~\big|~|a_k|\leq 3\text{m/s}^2,~ |\delta_k|\leq 1.22~\text{rad},|\delta_k-\delta_{k-1}|\leq 0.05~\text{rad}\}$, with initial steering condition $\delta^e_0 = 0$. Furthermore, the parameters for the $\cvar$ and confidence evaluation are set with $\alpha = 0.85$, $\theta_{\max} = 5$ and $\tau=1$. The safety loss function is chosen as $s(x,\xi) = -\|[\mathrm{x}^e,\mathrm{y}^e]^\top -[\mathrm{x}^o,\mathrm{y}^o]^\top\|^2$. The MPC horizon and the confidence dataset size are chosen respectively as $L=50$ and $m=30$. We used ForcesPro~\cite[Sec.~3]{FORCESNLP} to solve the MPC problem with a maximum of 500 iterations for the optimization solver.

\subsubsection{Baseline Algorithm}
We compared our \texttt{SIED-MPC} algorithm with two baseline approaches: Mean-MPC and DR-MPC, differences of which are summarized in Table~\ref{tab:baseline}. 
\begin{table}
    \centering
    \caption{Methods Comparison}
    \begin{tabular}{c||cc|c}\toprule
                   & Mean-MPC & DR-MPC      & Ours\\ \midrule
        Controller & MPC      & DR-MPC      & DR-MPC\\
        Filter     & EKF      & EKF         & SSIE \\
        $\theta_k$ & $\theta_{\max}$ & $\theta_{\max}$ & $\theta(F^\hatpi_k)$ \\ \bottomrule
\end{tabular}
    \label{tab:baseline}
\end{table}

For all approaches, the behavior model is set to the Constant Steering And Velocity (CSAV) model. Although this model predicts the obstacle behavior almost accurately during the straight sections of the route, it fails to account for deviations during the obstacle’s turn at the intersection, where the obstacle exhibits non-zero jerk and steering angle. Given the inherent randomness in the simulation, we conducted 20 simulation runs under identical conditions to ensure the reliability of our results.

\subsection{Evaluation}
\subsubsection{State Estimation Comparison}
We compare the state estimation performance between EKF and SSIE. The element-wise state estimation errors are plotted in Fig.~\ref{fig:StateEstimation}. 
\vspace{-20pt}
\begin{figure}[H]
    \centering
    \includegraphics[width = 0.8\columnwidth]{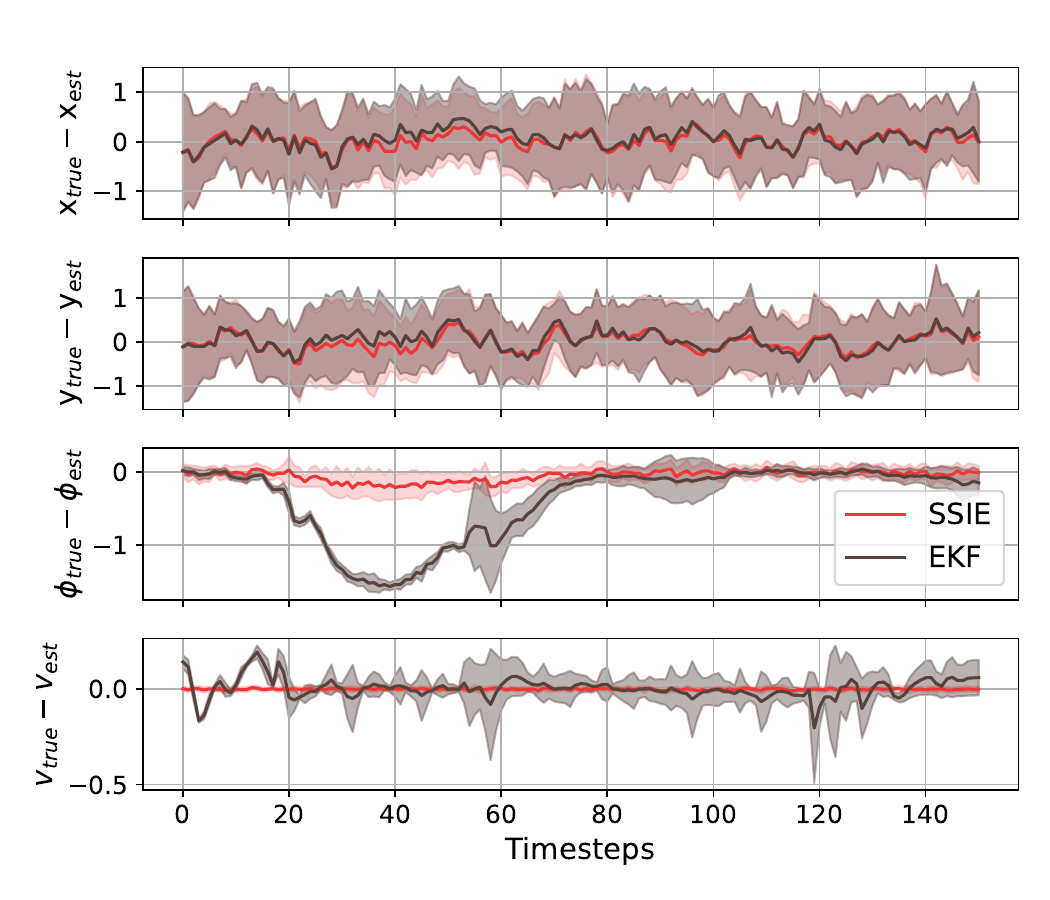}
    \vspace{-5pt}
    \caption{SSIE makes accurate estimations across all variables, whereas EKF makes errors during $t\in[20,80]$ when the input estimations deviate from the behavior model predictions. $x,y$ are measured in $m$, $v$ in $m/s$, and the heading angle $\phi$ in degrees.}
    \label{fig:StateEstimation}
\end{figure}
While EKF delivered reliable estimates of the obstacle vehicle positions, its lack of input (gap) estimation resulted in less accurate predictions of the vehicle's heading angle and velocity. In contrast, SSIE provided more accurate estimates for all obstacle state variables. 

To further understand the observed performance gap between EKF and SSIE, we analyze the tracking accuracy of SSIE's input estimations. The performance of SSIE’s input tracking is depicted in Fig.~\ref{fig:InputEstimation}, where we also illustrate the dynamic Wasserstein radius $\theta$, calculated based on the discrepancies between the behavior model’s outputs and the actual estimations. Notably, the radius size increases during the time interval $20 \leq t \leq 70$ and $t\geq 120$, where the input estimations either diverge from the behavior model’s predictions or its variance is large.
% \vspace{-10pt}
\begin{figure}
    \centering
    \includegraphics[width = 0.8\columnwidth]{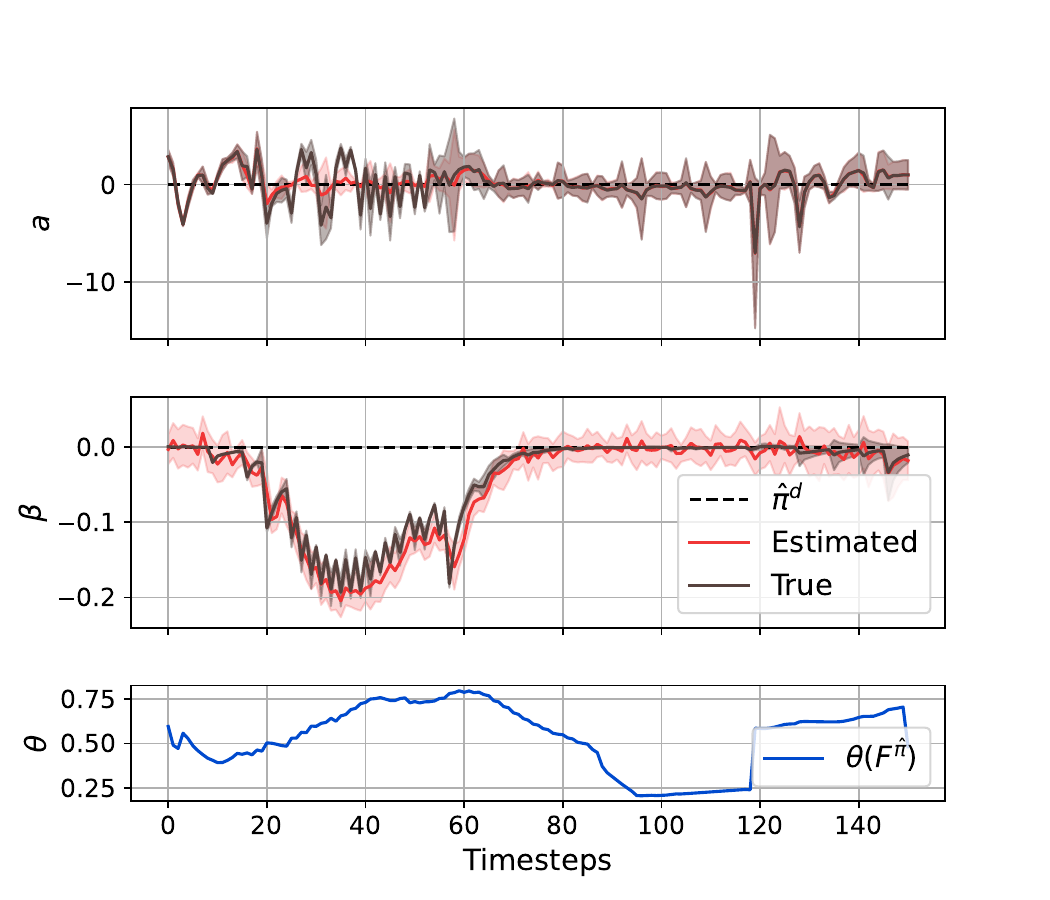}
    \caption{The CSAV model relatively predicts acceleration accurately, but not the slip angle. SSIE effectively tracks these deviations, as reflected in the dynamic adjustments of the Wasserstein radius $\theta$.}
    \label{fig:InputEstimation}
\end{figure}
\subsubsection{Safety Comparison}
While integrating estimation methods with the DR constraint provides additional safety margins that enhance safety, inaccuracies in state estimation and behavior model can be amplified within the MPC loop. We compare collision avoidance rates and computational complexity of the three methods in Fig.~\ref{fig:Computation}.
\vspace{-10pt}
\begin{figure}[H]
    \centering
    \includegraphics[width = 0.75\columnwidth]{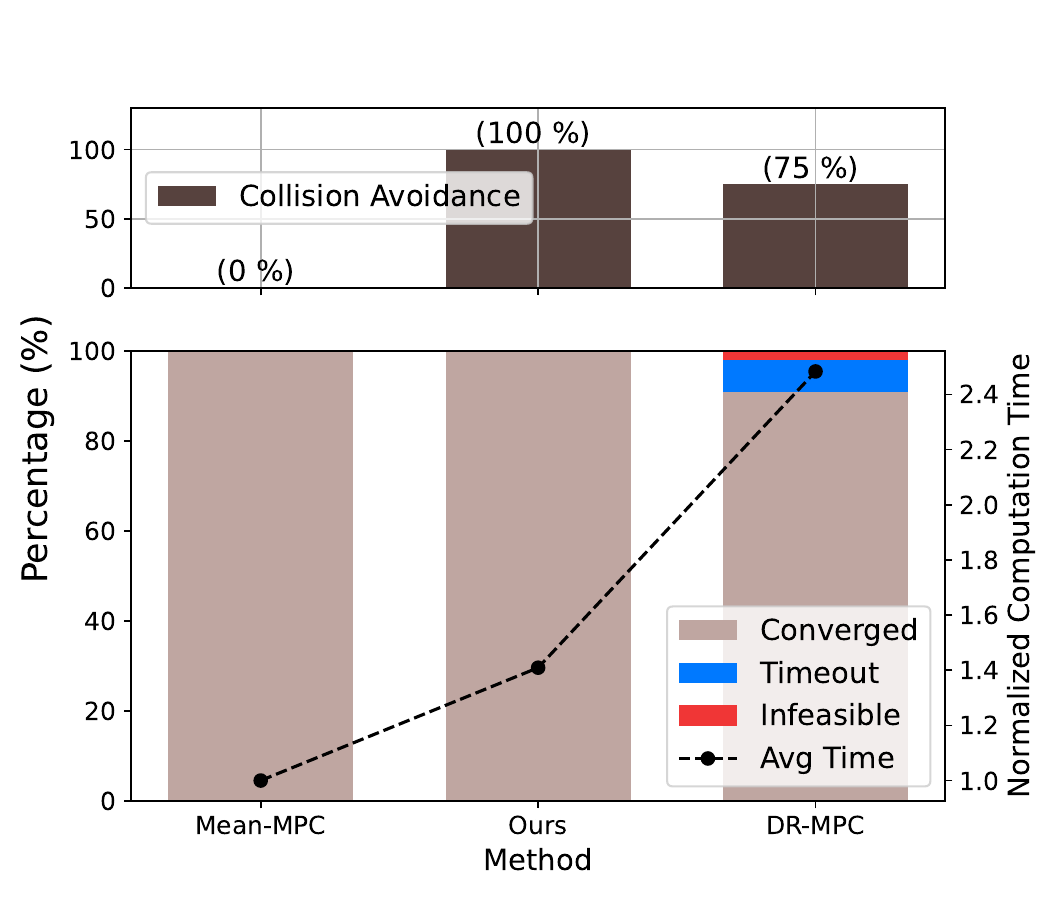}
    \caption{Mean-MPC failed to avoid collision in all episodes, whereas DR-MPC and our method achieved $75\%$ and $100\%$, respectively. To compare computational efficiency, computation times are normalized to that of Mean-MPC. The Mean-MPC had the shortest average computation time. Our algorithm required an additional $40\%$ computation time on average, while DR-MPC's computation time exceeded Mean-MPC's by $150\%$. }
    \label{fig:Computation}
\end{figure}

The collision avoidance rate for Mean-MPC shows that it lacked sufficient safety margins to avoid a noisy obstacle (estimation error) that behaves differently from the behavior model (prediction error). While the DR-MPC method had additional safety margins, its collision avoidance was compromised in multiple cases by inaccurate estimations and overly conservative constraints. Specifically, the constraints rendered the optimization problem either infeasible or necessitated more than the maximum iterations for convergence, indicating that adequate solutions could not be computed within the control interval\footnote{Physical time spent for per-iteration computation is described in~\cite[Sec.~4]{FORCESNLP}. We report a normalized computation time in Fig.~\ref{fig:Computation} for different computation and optimization configurations.}. In contrast, our method achieved a zero collision rate, while computations were made within the desired time intervals. 

Lastly, we evaluate the total cost over the predictive horizon at each timestep. As depicted in Fig.~\ref{fig:Cost}, the DR-MPC method reached significantly higher costs deviating from the desired ego vehicle states more than our method, primarily in an effort to avoid collisions. 
\vspace{-10pt}
\begin{figure}[H]
    \centering
    \includegraphics[width = 0.9\columnwidth]{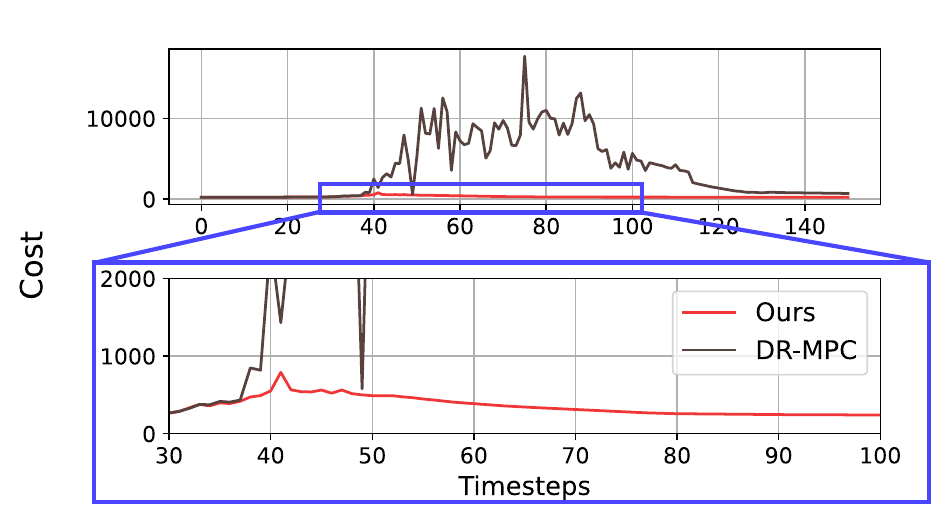}
    \caption{Cost evaluated at each time step. DR-MPC reported the average cost and its standard deviation as $3713.66\pm 3861.90$, in contrast to ours $295.08 \pm 99.66$. }
    \label{fig:Cost}
\end{figure}

\section{Conclusion}

In this work, we proposed a safe motion control framework that can address both the estimation bias and the predictive distribution shift caused primarily by the inaccuracies of a behavior model. We derived the SSIE algorithm to produce an unbiased obstacle state estimation and an optimal input gap estimation when only a nominal behavior model is given. Leveraging the results from the SSIE, we proposed a confidence-based DR-MPC problem to address the shift in predictive distribution that reasonably adjusts the conservatism of the control.The integrated \texttt{SIED-MPC} algorithm demonstrated improved collision avoidance in an autonomous driving scenario with enhanced state estimation accuracy and shorter computation.
%%%%%%%%%%%%%%%%%%%%%%%%%%%%%%%%%%%
%%%%%%%%%%%%%%%%%%%%%%%%%%%%%%%%%%%%

\bibliographystyle{IEEEtran} % We choose the "plain" reference style
\bibliography{refs} % Entries are in the refs.bib file

\end{document}